\pdfoutput=1
\documentclass[a4paper,12pt]{article}

\usepackage{cmap} 
\usepackage{mathtext} 
\usepackage[T2A]{fontenc} 
\usepackage[utf8]{inputenc} 
\usepackage[english]{babel} 
\usepackage{indentfirst} 
\usepackage{extsizes} 
 \usepackage{geometry} 

\usepackage{setspace} 
\usepackage{enumitem} 
\setlist{leftmargin=25pt} 
\usepackage{multicol} 
\usepackage{soulutf8} 
\usepackage{tocloft}
\tocloftpagestyle{main}

\newcommand{\CourseDate}{}

\usepackage{titleps}
\newpagestyle{main}{
\setheadrule{0.4pt}
\setfootrule{0.4pt}
\setfoot{ \CourseDate}{}{\thepage}
}
\usepackage{hyperref}
\usepackage[usenames,dvipsnames,svgnames,table,rgb]{xcolor}
\hypersetup{
unicode=true, 
colorlinks=true, 
linkcolor=black!15!blue, 
citecolor=green, 
filecolor=magenta, 
urlcolor=NavyBlue, 
}

\usepackage{mleftright}
\usepackage{tikz}
\usetikzlibrary{arrows.meta,bending,chains}
\usepackage{longtable}
\usepackage{tikz-cd}
\usepackage{lscape}
\usepackage{amsmath}
\usepackage{amssymb}
\usepackage{amsthm}
\usetikzlibrary{shapes,arrows,intersections}
\usetikzlibrary{matrix,fit,calc,trees,positioning,arrows,chains,shapes.geometric,shapes}
\usepackage{setspace}
\usepackage{esint}
\usepackage{tikz}
\usetikzlibrary{decorations.pathreplacing,calc}

\usepackage{graphicx}

\newenvironment{lemma}[2][Lemma]{\begin{trivlist}
\item[\hskip \labelsep {\bfseries #1}\hskip \labelsep {\bfseries #2.}]}{\end{trivlist}}

\newenvironment{corollary}[2][Corollary]{\begin{trivlist}
\item[\hskip \labelsep {\bfseries #1}\hskip \labelsep {\bfseries #2.}]}{\end{trivlist}}

 \usepackage{enumitem}
 \usepackage{xcolor}

\usepackage{graphicx}   
\usepackage{booktabs}   
\usepackage{caption}

\usepackage{tikz}
\usepackage{pgf-pie}

\usepackage{tikz}
\usetikzlibrary{
  shapes.geometric,
  arrows.meta,
  positioning
}

\usepackage{xcolor}
\usepackage{tikz}
\usetikzlibrary{
  shapes.geometric,
  shadows,
  arrows.meta,
  positioning
}
\usepackage[x11names,dvipsnames]{xcolor}
\usepackage{tikz}
\usetikzlibrary{shadows,arrows.meta}
\usepackage[dvipsnames]{xcolor}
\usepackage{float}

\usepackage{booktabs}
\usepackage{tabularx}
\usepackage{xcolor}
\usepackage{makecell}

\usepackage[utf8]{inputenc}

\usepackage[dvipsnames]{xcolor}
\usepackage{tcolorbox}
\tcbuselibrary{breakable}
\usepackage{tikz}
\usepackage{pgf-pie}
\usepackage{booktabs} 
\usepackage{url}       
\usepackage{hyperref}  

\begin{document}
    \newpage
    \hypertarget{intro}{}
    \newpage
    \linespread{1}
    \selectfont

\title{\Huge\bfseries EEFSUVA: A Mathematical Olympiad LLM Benchmark} 
\author{Nicole N. Khatibi, Daniil A. Radamovich, Michael P. Brenner }
\date{}

\maketitle

\begin{abstract}

Recent breakthroughs have spurred claims that large language models (LLMs) match gold medal Olympiad to graduate-level proficiency on mathematics benchmarks. In this work, we examine these claims in detail and assess the extent to which current benchmarks capture genuine LLM mathematical reasoning. The composition of these benchmarks, primarily drawing from the International Mathematics Olympiad (IMO) and related  competitions, may overstate models’ reasoning ability due to potential data contamination and a narrow focus on familiar problem types. To enable a more holistic assessment of mathematical understanding, we introduce EEFSUVA, a novel benchmark curated from under circulated regional and national Olympiads of Eastern Europe and the countries from the former Soviet Union. These contests feature problems of comparable difficulty to the IMO and are renowned for demanding nonstandard problem-solving techniques, yet their problems are far less prevalent in online corpora. Preliminary results suggest that even state-of-the-art LLMs exhibit a notable performance decline on EEFSUVA relative to other Olympiad-style benchmarks. These findings also suggest the potential importance of broader evaluation datasets for a fuller assessment of mathematical reasoning and for guiding future model development.

\end{abstract}

    \section{Introduction}
\noindent The emergence of large language models (LLMs) has profoundly impacted numerous scientific disciplines, including mathematics. Within the mathematical community, there is a growing discourse regarding the capabilities of LLMs, with some researchers claiming that these models achieve problem-solving efficiency comparable to early-year graduate students. Beyond pure performance, there is also a popular perception that these LLMs serve as powerful tools for individual mathematical development and success.\\

\noindent  A quick and critical examination of existing benchmarks and platforms specializing in LLM evaluations reveals an interesting characteristic. Many existing benchmarks often rely on problems from a relatively narrow range of competitions.  While these problems are undeniably challenging, they represent only a subset of advanced mathematics, leaving other problem types underrepresented. Incorporating a broader spectrum would support a more comprehensive evaluation of mathematical reasoning. Our work addresses this gap by introducing a novel benchmark consisting of thirty-nine problems primarily from mathematics Olympiads of Eastern Europe and countries from the former Soviet Union and six problems from various mathematical texts specifically by Vladimir Arnold. The problems cover topics such as combinatorics, algebra, and number theory, but do not include geometry. The problems selected for this benchmark share a crucial commonality, they consistently demand an extra, nontrivial step in mathematical reasoning, often requiring novel thinking strategies beyond mere computation to pattern recognition. This main characteristic is exactly what makes them particularly suitable for unraveling and probing at the limitations of current LLMs in mathematical comprehension. The construction of this benchmark involved an extensive and meticulous search for specific problem sources. Our searching process led us to identify Eastern European and countries of the former Soviet Union mathematic regional and national Olympiads as particularly valuable. The problems, often less widely circulated than their Western counterparts, proved exceptionally effective for benchmarking due to two primary reasons, their relative obscurity minimizes the risk of LLM pre-training exposure, and their inherent difficulty requires truly novel problem solving approaches. While previous benchmarks often relied on heavily explored competition problems from the IMO and related competitions, our approach intentionally sought sources known for their equally challenging national and regional competitions from other traditions. The integrity of our evaluation is further enforced by the unambiguous nature of the problem solutions. Which is to be further discussed in the following sections.

    \section{Related Work}

\noindent A leading resource for benchmarks based on mathematics competitions and Olympiads is MathArena, which provides a comprehensive overview of existing datasets and model performance. Our benchmark was curated with MathArena as a primary reference point, as the platform systematically presents both competition sources and large language model results. However, after analyzing the scope and problem sets emphasized in MathArena, we deliberately shifted our focus toward lesser-known mathematics Olympiads and competitions, aiming to broaden the diversity of our benchmark and reduce overlap with existing evaluations. Drawing on prior experience in mathematics Olympiads at both the national and international levels, we were able to identify clear directions for sourcing problems for our benchmark. \\

\noindent MathArena has already conducted evaluations of LLMs on the 2025 IMO, widely regarded as the standard for the most challenging mathematics competition. In their evaluation, GPT-5 achieved a 38.10\% accuracy rate. Using our benchmark, however, GPT-5 Thinking obtained a slightly lower score of 35.89\%, suggesting that our benchmark yields comparable outcomes while capturing a different problem selection than MathArena. In contrast, Gemini 2.5 Pro reached 31.55\% accuracy on MathArena, but achieved 0\% accuracy on our benchmark, underscoring the comparable difficulty and different emphasis of our dataset. It is worth noting that, according to MathArena’s documentation, GPT-5 was released after the 2025 IMO, raising the possibility of data contamination. By comparison, nearly all of the problems in our benchmark are at least five years old, meaning they were potentially available in model pretraining corpora. Nevertheless, GPT-5-Thinking solved only about one-third of them. This outcome suggests that the limited performance cannot be attributed to recency or contamination but rather reflects the intrinsic reasoning difficulty of the specific Olympiads we chose and the problems sourced from them. Moreover, among the various competitions reported on MathArena, which include the Harvard-MIT Math Tournament (HHMT), Stanford Math Tournament (SMT), and the Brown University Math Olympiad (BRUMO), the IMO produced some of the lowest accuracy rates across models, yet our benchmark reduced performance even further. This suggests a potential limitation of relying solely on specific competitions as the basis for benchmarking, since it may not fully capture the range of reasoning challenges posed by other mathematical Olympiads and competitions.

\begin{table}[h]
\centering
\begin{tabular}{lcccccc}
\toprule
Model & HHMT & SMT & BRUMO & IMO &USMO & EEFSUVA \\
\midrule
Gemini 2.5 Pro      & 82.50\% & 84.91\%  & 90.00\% & 31.55\% &24.40\% &0\% \\
GPT-5 Thinking/High & 88.33\%  & 91.98\% & 91.67\% & 38.10\% &n/a & 35.89\%  \\
\bottomrule
\end{tabular}
\caption{Comparison with EEFUSVA}
\end{table}

It is important to note, as stated on MathArena, that HHMT, BRUMO, and IMO all occurred prior to the release of GPT-5, raising the possibility of data contamination in the reported results for this model. In the case of Gemini 2.5 Pro, similar concerns arise for HHMT, since the model was released after the competition date. By contrast, our benchmark was constructed from Olympiads and competitions dating back at least five years. This makes contamination more likely, since such problems would almost certainly have been present in the model’s training data. However, if that were the case, one would expect the accuracy on our benchmark to be higher rather than lower. This contrast suggests that the range of competitions commonly used in benchmarks, though valuable, reflects only part of the mathematical landscape. These contests are unquestionably difficult, but they represent only one tradition of problem solving, leaving other competitions of comparable rigor less visible. Benchmarks such as those in MathArena provide a strong standard, yet one might anticipate that our benchmark, drawing on older, non-Western problems, would overlap more with pre-training corpora and thus yield higher accuracy. This potentially suggests that the difficulty of non-Western competitions is under recognized and the reliance on a limited benchmark set may give an incomplete picture of model capabilities. Broadening evaluation to include a wider range of mathematical reasoning challenges could therefore provide a more balanced assessment and guide further progress in LLM development.

    \section{Problem Curation}

\begin{figure}[H]
  \centering
  \begin{tikzpicture}[
      scale=0.6,                   
      transform shape,             
      process/.style={
        rectangle,
        font=\bfseries,
        draw=gray!70,
        fill=#1!50,
        text=black,
        rounded corners=6pt,
        minimum width=5cm,
        minimum height=1.0cm,
        inner sep=6pt,
        align=center,
        drop shadow={
          shadow xshift=1pt,shadow yshift=-1pt,
          fill=black!20,rounded corners=6pt
        }
      },
      arrow/.style={
        thick,
        -{Latex[length=2mm]},
        black!70
      }
    ]
    \node[process=Blue]   (s1) at (0,0)    {Check what Olympiads\\have been used the most};
    \node[process=TealBlue]   (s2) at (4,-2)   {Search for Eastern\\European and former SU Olympiads (older)};
    \node[process=Orange] (s3) at (0,-4)   {Extract any sources with\\numerical answers first};
    \node[process=Purple] (s4) at (4,-6)   {Solve each problem and check correctness};
    \node[process=Green]  (s5) at (0,-8)   {Run the LLM for a solution};
    \node[process=Red]    (s6) at (4,-10)  {Evaluate result};

    \draw[arrow] (s1) -- (s2);
    \draw[arrow] (s2) -- (s3);
    \draw[arrow] (s3) -- (s4);
    \draw[arrow] (s4) -- (s5);
    \draw[arrow] (s5) -- (s6);
  \end{tikzpicture}
  \caption{Workflow for selecting Olympiad problems}
  \label{fig:workflow-zigzag}
\end{figure}
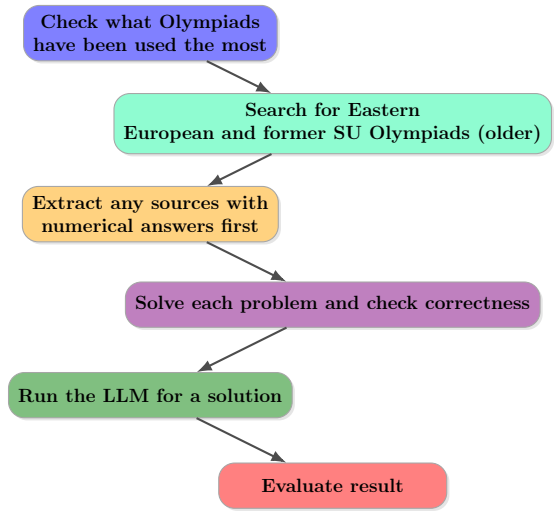

\subsection{Olympiads}
\noindent The problems used to curate this benchmark were found in the following way. A quick search of previous benchmarks showed that the United States of America Mathematical Olympiad (USAMO) and the IMO were heavily used to evaluate these LLMs. These sources were very popular which meant that their problems were already heavily utilized. This made even going back to older years of the competitions fruitless. Because of this, we  decided to turn our attention to lesser-known, yet highly competitive Olympiads from countries from the former Soviet Union and Eastern European countries. These mathematical Olympiads are known in the mathematics community to be extremely competitive, with problems ranging from regional to state to national levels. This opened a new pathway to potentially more difficult problems. Yet, their main attraction was not even their difficulty, it was their unpopularity. They were not as widely circulated in the public eye, which made them highly valuable in our search. \\

\noindent Knowing now where to look, the issue began in the hunt for finding the specific Olympiads. Their underexposure also meant the resources would be scarce. Although most competitions were not easily tracked down, deep searching led us to many valuable sources providing problems from competitions dating back five to forty years. Typical Olympiad competitions feature a wide variety of problem types and topics, including both proof-based and numerical problems in areas such as number theory, geometry, and algebra. We focused exclusively on problems with numerical solutions, avoiding proof-based problems. Many proof-based questions involve well-known claims or standard arguments that are easily recognized by LLMs, increasing the likelihood of correct responses through pattern recognition rather than reasoning. In addition, proof-based problems require additional steps for verification and often depend on human judgment. To avoid the additional time and subjectivity involved, we focused exclusively on problems that could be evaluated unambiguously. Specifically, we prioritized problems with clear numerical answers that could be verified automatically and in seconds. These main objectives made finding the exact sources that we needed an additional challenge. Most Olympiad competitions contain more proof-based problems than computational ones, so we were typically able to extract at maximum two suitable problems per competition. However, in many cases, we were lucky to find even one viable problem. However, despite these challenges, we were able to identify a substantial number of high-quality problems, and our hypothesis to shift our focus to countries from the former Soviet Union and Eastern European mathematical Olympiads and competitions deemed to be correct.



\subsection{V. I. Arnold Work}
\noindent Alongside the curated problems from various mathematical Olympiads and competitions, we also included six  problems drawn from works by Vladimir I. Arnold, including "Ordinary Differential Equations", "Mathematical Methods of Classical Mechanics", and "The Mathematical Trivium". The two books included solutions for selected problems, and our search followed the same approach for the mathematical competitions, we focused only on problems with numerical answers. Additionally, The Mathematical Trivium, an exam consisting of 100 problems to be solved within three hours, was another source. The Trivium was designed for advanced undergraduate and graduate students, with problems that draw on higher-level topics while retaining a largely computational format. Although not purely exercises in computation, these problems emphasize technical problem-solving in advanced areas of mathematics. While the Trivium has never had any official solutions published or verified in any text, it has long served as a recreational challenge among mathematicians. Among the first few problems in the Trivium, which is included in our benchmark, one initially appears to be a straightforward limit evaluation but turns out to be quite difficult without a clever approach. For our purposes, we solved and verified only the initial handful of problems, as the full Trivium set was not our primary focus, and remains to potentially be a topic for later work. The benchmark emphasizes Olympiad-style reasoning, and Trivium was selected primarily to test computational ability in a complementary way. We chose to include problems from Arnold’s works as part of our curation strategy, given his reputation as one of the greatest polymath mathematicians. His texts pose many challenging problems, and given their popularity, we expected LLMs to perform well on them. However, we still found several problems that the models failed to solve correctly.



\noindent --------------

\noindent Each problem selected from the Olympiads and Arnold’s texts was first solved independently, after which the  proposed solutions were consulted to verify correctness. This ensured that our solutions were both self-derived and consistent with the established sources. This process not only verified the correctness of our own solutions but also served as a check on the validity of the proposed solutions themselves. After establishing the correctness of each solution, we submitted the problems to the LLMs and evaluated their responses relative to the verified solutions. Our narrowed focus on numerical problems made the initial evaluation straightforward. Often a simple 'yes' or 'no' that could be quickly confirmed using the provided solution. Additional layers of verification were also applied, which will be explained in more detail in the following sections.

\subsection{Problem Types}

\noindent This benchmark consists of two main sources of problems, as mentioned before, the various different mathematic competitions, exams, and Olympiads, and the various works from Vladimir I. Arnold. The majority of the problems taken from the various competitions as seen in figure a) were from combinatorics and number theory, where as in figure b) from the texts the majority was taken from Mathematical Methods for Classical Mechanics.

\subsubsection*{Problem Category and Source Breakdown }

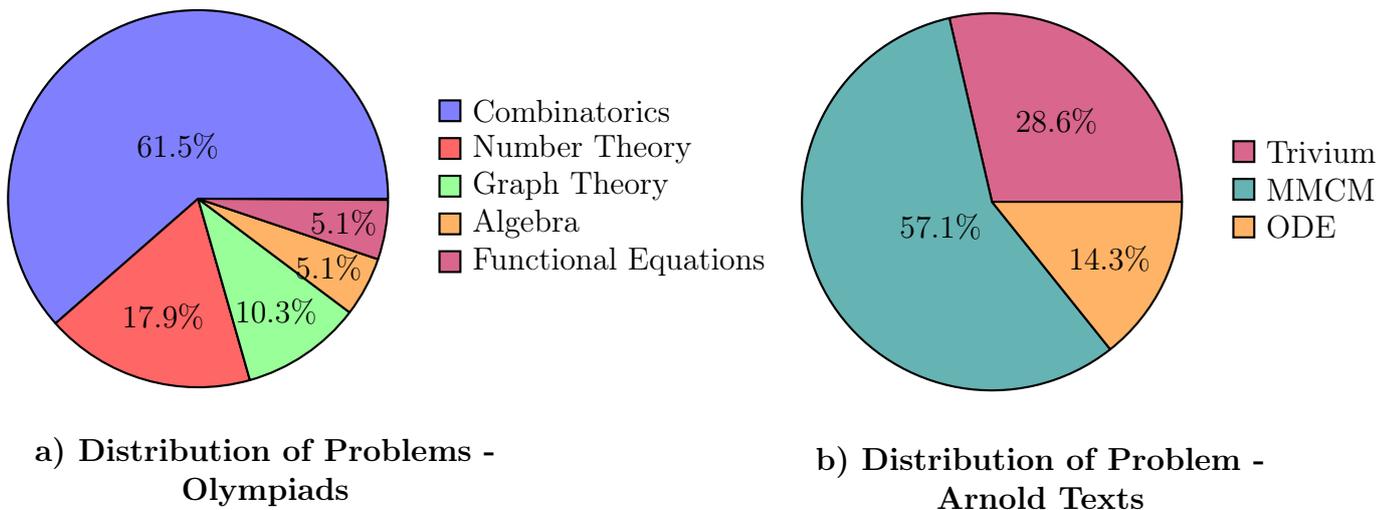
\begin{center}
\begin{minipage}{0.40\textwidth}
\centering
\begin{tikzpicture}
\pie[text=legend, radius=2.5, color={blue!50, red!60, green!40, orange!60, purple!60}]{
61.5/Combinatorics ,
17.9/Number Theory,
10.3/Graph Theory ,
5.1/ Algebra,
5.1/Functional Equations
}
\end{tikzpicture}

\vspace{0.7em}
\textbf{a) Distribution of Problems - Olympiads}
\end{minipage}
\hfill
\begin{minipage}{0.40\textwidth}
\centering
\begin{tikzpicture}
\pie[text=legend, radius=2.5, color={purple!60, teal!60, orange!60}]{
28.6/Trivium,
57.1/MMCM ,
14.3/ODE 
}
\end{tikzpicture}

\vspace{0.7em}
\textbf{b) Distribution of Problem - Arnold Texts}
\end{minipage}
\end{center}

 \subsubsection{Combinatorics}
\noindent  Combinatorics has a central role in mathematical Olympiads due to its blend of creativity, logic, and minimal prerequisite knowledge. Although parts of the broader field of combinatorics intersect with algorithmic techniques, Olympiad-style combinatorics is intentionally non-algorithmic. These problems almost never collapse to standard procedures and are often designed to frustrate  standard approaches. Many lack equations to manipulate, figures to analyze, or even numbers to compute. Consequently, they require flexible reasoning that entails not only clever constructions but also the invention of new structures, rather than reliance on the recognition of familiar patterns.
LLMs are generally trained to identify and complete patterns based on seen data, but in Olympiad combinatorics, the structure must often be created on the spot, not retrieved from prior examples. These problems typically involve multistep chains of reasoning, where each step depends delicately on the previous one. This makes them highly sensitive to even small logical errors. However, LLMs frequently generate hallucinated intermediate steps and introduce logical inconsistencies, thereby producing solutions that are incorrect or incomplete. It is precisely these characteristics that make combinatorics account for such a significant portion of the benchmark.

 \subsubsection{Number Theory}
\noindent  Olympiad-style number theory problems are a combination of pure mathematics with formal logic, often making them some of the most deceptively difficult challenges for both humans and LLMs. Although number theory appears algorithmic, since it draws on familiar topics like modular arithmetic and Diophantine equations, the Olympiad variant is rarely computational. Instead, it emphasizes clever manipulation and substitutions with unexpected usage of elementary tools in non-standard ways, similar to that of combinatorics problems. Many number theory solutions involve a vital moment in which the initially obscure problem becomes instantly simplified through a key observation or trick. For LLMs, these problems are particularly challenging because they often cannot be solved by surface-level pattern recognition. The mistakes of LLMs in number theory problems are very similar to those of combinatorics problems as well. Typically, errors arise from hallucinated steps or the failure to recognize key simplifications. Owing to its similarity in character to combinatorics problems, number theory exhibits a disproportionately high error rate compared to other topics in our benchmark.
    
\section{Evaluation}

\noindent Evaluating the effectiveness of the LLMs in solving the curated problems was a multistep process. The first stage involved independently analyzing and solving each problem and constructing an outlined solution, which provided a skeletal outline of the proposed answer. Next, we verified the accuracy of the solutions provided by the original sources. As noted earlier, most sources included proposed solutions which allowed us to confirm the validity of both the problems and their corresponding answers. Our outlines were compared against the proposed solutions and the sources were thus verified. Once this verification step was completed, we submitted each problem to the LLMs specifically, Gemini 2.5 Pro and Chat GPT 5 Thinking using a brand-new chatbot session for every single problem to prevent any cross contamination of context. After several minutes, the model would output a numerical answer, which we then compared against the verified solution. In order to fully mark a problem as incorrectly solved, we ran each problem in a new chat session twice. This was done so because in some cases the LLM produced an incorrect answer on the first attempt, however, occasionally on the second attempt it produced the correct result. In particular, we never encountered a case in which the model produced a correct solution on the first attempt and an incorrect one on the second. Below are the results from the two leading commercial LLMs  we conducted tests on. 
\\\\

\begin{table}[H]
\centering
\renewcommand{\arraystretch}{2.5}
\resizebox{\textwidth}{!}{%
\begin{tabular}{|l||c|c|c|c|c|c|c|}
\hline
\textbf{Model} & \textbf{Overall} & \textbf{Combinatorics} & \textbf{Number Theory} & \textbf{Graph Theory} & \textbf{Algebra} & \textbf{Functional Equations} & \textbf{Arnold} \\
\hline
\textbf{GPT-5 Thinking} & 35.89 & 25.00 & 42.85 & 25.00 & 100.00 & 0.00 & 28.57 \\

\textbf{Gemini 2.5 Pro} & 0 & 0 & 0 & 0 & 0 & 0 & 0 \\
\hline
\end{tabular}
}
\caption{\textbf{Pass Rates by Model and Question Type}}
\label{Table:eval_results}
\end{table}

\noindent The percentage rates for correctly solved problems from the evaluation are shown above. The emergence of GPT-5 led to a significant performance increase over Gemini 2.5 Pro, which had an overall zero percent pass rate. Overall, with a pass rate of approximately 36\%, it is clear that the new advancements still require further development. \\

\noindent In our evaluation process, it is also worth noting that we paid very close attention to the reasoning and solution steps for the test problems. While our primary criterion was the correctness of the numerical answer, the reasoning of the LLM shed light on the potential reliance on pattern recognition. Our close attention to the LLM's thought process was additionally in an attempt to see if the model picked up on the fact that the problems and some proposed solutions were open sourced.  When the model encountered a prompt that resembled for example, a known IMO–shortlist problem, it seemed to anchor heavily on that reference. The LLM seemed to import the familiar answer structure it had been previously trained on and proceeded as if the current task were the same. In multiple instances, we believed it probable that the superficially similar source problem differs in a small but crucial way, and that deviation is enough to make the transplanted solution schema invalid. The model most likley recognizes the problem \emph{type} but, biased by its training on the familiar variant, cannot adapt its predetermined reasoning to the altered constraints and therefore produces an incorrect result. In effect, it retrieves a cached solution rather than constructing a fresh argument from first principles. This is precisely why, once the technical details diverge, the LLM struggles to reconcile the mismatch. This behavior suggests that heavy exposure to IMO-style material produces a brittle competence, strong on memorized templates yet highly vulnerable to small structural changes. Below highlights an important example of this.

\medskip

\begin{tcolorbox}[breakable, colframe=red!100, colback=red!5, title={Gemini 2.5 Pro Reasoning}]

\textbf{Problem:} Let \(k\) be a positive integer. Determine the least integer \(n \ge k + 1\) for which the following game can be played indefinitely: \\

Consider \(n\) boxes labelled \(b_1, b_2, \dots, b_n\). Initially, box \(b_i\) contains exactly \(i\) coins for each \(i=1,2,\dots,n\). At each step one performs:\\

A. Choose \(k+1\) boxes.\\
B. Of these \(k+1\) boxes, select \(k\) of them and remove at least half of the coins from each; then add to the remaining box \(b_i\) exactly \(i\) coins.\\
C. If any box is empty, the game ends; otherwise, repeat from step 1.

\medskip

\textbf{LLM Answer:} \[
n =
\begin{cases}
k + \left\lfloor \dfrac{2}{k} \right\rfloor + 1, & k \ge 2,\\[6pt]
3, & k = 1.
\end{cases}
\]

\textbf{LLM Reasoning:} "This piecewise definition seems to be the correct result based on our analysis. For a contest setting, a single elegant formula is usually expected, which suggests there might be a more subtle argument. However, based on the provided logic, this is the derived answer".\\\\
"Given the case-by-case success of the lower bounds, the answer appears to be piecewise. However, the value \( n = k + \left\lfloor \frac{k}{2} \right\rfloor + 1 \)
 is the value from the official IMO Shortlist, which suggests that my failure analysis for k=2,3 or my subsequent analysis of player strategy might be flawed". \\

\medskip
{\color{blue}\textbf{Note:} In all of our testing we never mention from where the problems were sourced from or which competition it came from, the problems were sent as is. The LLM recognizes a similarity between the problem we provided and a problem it encountered during training, however, the two are not identical. The model notices that this may be pertaining to some mathematical contest, however it cannot identify which one, but it believes it may be from an IMO. Instead of solving the problem or running an internet search, it bases the structure of the solution on the recognized IMO problem instead of solving it properly. Once there is a detection of a problem similar to problems fed during training, the model immediately relies on the problem it knows, even though the two are not the same and are not from the same competitions. }
\end{tcolorbox} 

 \

 \noindent It is also important to note that the reasoning process employed by LLMs significantly influences their performance on our benchmark. Specifically, models such as Gemini 2.5 Pro seem to  adopt one of two main strategies when addressing a given question: drawing upon their internal training data (i.e., learned knowledge) or conducting a live web search. When the model encounters a question resembling material it has already learned, it is more likely to bypass live search and rely solely on its training. While this approach is efficient in many cases, it becomes less effective when the model’s training coverage is limited. By focusing training predominantly on a limited set of challenging mathematics competitions, we risk overlooking other competitions of comparable difficulty that are equally worthy of inclusion. As noted above, minor changes in problem structure led the LLM to assume it already knew the appropriate steps to apply. However, when the model encountered the singularity, it became evident that it did not possess a clear strategy for proceeding. This leads us to another critical standpoint, is the model unable to solve the problem generally due to a lack of exposure, or is the model not noticing that the question and some proposed solution exist on the open web. Both of these pose an issue, however if the problem and a proposed solution are available with a search, the greater issue becomes why did the model not notice it when if it ran into issues using pretrained knowledge.  

    \section{Conclusion}

\noindent The distinctive aspect of this benchmark is that every problem was sourced directly from the open web. Our selection process was guided by a specific hypothesis: many existing benchmarks appear to follow a narrow pattern, drawing predominantly from the IMO and other well-known Western competitions. While these benchmarks are valuable, their concentrated focus may help explain some of the performance patterns observed in LLMs. Motivated by this observation, we sought to design a benchmark that could probe model capabilities from a different angle. Several factors contributed to the behavior we observed: the models’ limited prior exposure to the Olympiads we selected, their relative underrepresentation in public datasets, the inherent reasoning tendencies of LLMs, and the high difficulty level of the chosen problems. Although the problems in this benchmark may appear simple because of their deterministic, numerical form, they nevertheless posed significant challenges for the models and revealed important gaps in current capabilities.
    \newpage
    
\makeatletter
\renewcommand\@biblabel[1]{}
\makeatother

    \appendix
\section{Appendix} \label{sec:appendix}

\newtheorem{proposition}{Proposition}

\subsubsection{Detailed solution of Combinatorics problems}

More than half of the selected Olympiad problems were combinatorical problems therefore, we give an example from our benchmark with a full solution below. This problem is one that Gemini 2.5 Pro failed to solve correctly however GPT-5 Thinking did solve correctly. 

\medskip

\begin{tcolorbox}[breakable, colframe=blue!60!yellow, colback=blue!5!white, title=Sample Olympiad Combinatorics Problem and Full Solution]
\textbf{Problem:} 
A ten-level \(2\)-tree is drawn in the plane: a vertex \(A_1\) is marked, and it is connected  
by segments with two vertices \(B_1\) and \(B_2\). Each of \(B_1\) and \(B_2\) is connected by  
segments with two of the four vertices \(C_1, C_2, C_3, C_4\) (each \(C_i\) is connected with  
exactly one \(B_j\)); and so on, up to \(512\) vertices \(J_1, \ldots, J_{512}\). Each of the vertices  
\(J_1, \ldots, J_{512}\) is coloured blue or golden.

\vspace{1em}

Consider all permutations \(f\) of the vertices of this tree, such that:
\begin{enumerate}
    \item If \(X\) and \(Y\) are connected with a segment, then so are \(f(X)\) and \(f(Y)\),
    \item If \(X\) is coloured, then \(f(X)\) has the same colour.
\end{enumerate}

Find the maximum \(M\) such that there are at least \(M\) permutations with these properties,  
\emph{regardless of the colouring}.

\newcommand{\thinline}{\noindent\rule{\linewidth}{0.2pt}}

\thinline

\textbf{Solution:}
The answer is: $2^{2^7}$

First we need a suitable terminology. Similarly to a 10-level $2$-tree, we can define a $k$-level $2$-tree for $k \geq 1$.
For convenience we suppose that all the segments between vertices are directed from a letter to the next
one. The number of the letter marking a vertex we call the \emph{level} of this vertex; thus $A_1$ is the only vertex
of level 1, $B_1$ and $B_2$ belong to level 2, and so on. We will also call \emph{descendants} of a vertex $X$ all vertices
which can be reached from $X$ by directed segments.

Let $T_1$ and $T_2$ be two $k$-level $2$-trees with coloured leaves. We call a bijection 
$f : T_1 \to T_2$ an \emph{isomorphism} when two conditions are satisfied:
\begin{enumerate}
  \item[(i)] if two vertices $X$ and $Y$ are connected by an edge in $T_1$, then $f(X)$
and $f(Y)$ are connected by an edge in $T_2$;
  \item[(ii)] if $X$ has some colour in $T_1$, then $f(X)$ has the same colour in $T_2$.
\end{enumerate}
When $T_1 = T_2$, we call $f$ an \emph{automorphism} of the tree.  
By $\chi(k)$ we denote the minimal number of automorphisms a $k$-level $2$-tree with coloured leaves can have
(the minimum is taken over all colourings).  
Our problem is to find $\chi(10)$.

\begin{lemma}\\
Isomorphism of trees preserves the level of a vertex.
\end{lemma}

\begin{proof}
An isomorphism $f$ cannot diminish the degree of a vertex. Indeed, neighbours of each vertex $X$
become neighbours of $f(X)$, therefore the degree of $f(X)$ is not less than the degree of $X$.
By the pigeonhole principle it also means that the degree cannot increase.
It follows that the last-level vertices go to the last-level vertices. Therefore vertices of the previous level go to the same level,
since they remain neighbours of the last-level vertices, and so on.
\end{proof}

Now we are ready to solve the problem.\\
\textbf{First proof of the lower bound, by induction.}

\begin{proposition}
For each $k \geq 2$ we have
\[
\chi(k) \geq (\chi(k-1))^2.
\]
\end{proposition}

\begin{proof}
In a $k$-level tree the descendants of $B_1$ (including $B_1$) form a $(k-1)$-level tree $T_1$.
This graph has at least $\chi(k-1)$ different automorphisms. The same is true for the tree $T_2$
formed by the descendants of $B_2$.  
Let $g$ and $h$ be automorphisms of $T_1$ and $T_2$ respectively.  
Now we can define a mapping $f$ of the whole tree, applying $g$ to descendants of $B_1$, $h$ to descendants of $B_2$, and $A$ to itself.  
Obviously $f$ is an automorphism: for $X = A$ the condition holds since $B_1$ and $B_2$ were mapped to themselves (by Lemma 1),
and for $X$ in $T_1$ or $T_2$ because $g$ and $h$ are automorphisms.  
Thus for each pair $(g,h)$ there is an automorphism $f$, different pairs produce different $f$, and the number of pairs is at least $(\chi(k-1))^2$.
\end{proof}

\begin{corollary}\\
For $k \geq 3$ we have
\[
\chi(k) \geq 2^{2^{k-3}}.
\]
\end{corollary}

\begin{proof}
This inequality is proved by induction, with Proposition 1 as the induction step.  
It remains to check it for $k = 3$.  
If in a $3$-level $2$-tree at least one of the vertices $B_1, B_2$ has two descendants of the same colour,
there is an automorphism exchanging these two vertices and preserving the rest.  
If each of $B_1, B_2$ has one blue and one golden descendant,
there is an automorphism exchanging $B_1$ and $B_2$ and preserving colours of their descendants.  
In both cases the number of automorphisms (including the identical one) is at least $2$.
\end{proof}

\medskip

\noindent
\textbf{Second proof of the lower bound (without induction).}  \\
We already know that every $3$-level $2$-tree with four coloured leaves has at least two colour-preserving automorphisms.  
Now every $n$-level tree, $n \geq 3$, has $2^{n-3}$ vertices of level $n-2$, and the descendants of each of these vertices form a $3$-level tree.  
It is enough to consider automorphisms preserving vertices of level $n-3$ (and, a fortiori, of all lesser levels).  
Such an automorphism can act on the descendants of each of $2^{n-3}$ vertices of level $n-2$ in at least $2$ ways.  
Thus there are at least $2^{2^{n-3}}$ such automorphisms.

It remains to construct for each $k \geq 3$ a colouring of a $k$-level tree
admitting exactly $2^{2^{k-3}}$ automorphisms. As it happens sometimes, we will prove somewhat more.

\begin{proposition}
For each $k \geq 3$ there are three colourings $M_1, M_2, M_3$ of leaves of a $k$-level $2$-tree
such that the trees with these colourings are not isomorphic, and each of these colourings admits exactly $2^{2^{k-3}}$
automorphisms.
\end{proposition}

\begin{proof}
For $k = 3$ let $C_1, C_2$ be the descendants of $B_1$, and $C_3, C_4$ the descendants of $B_2$.  
The three colourings are the following: 
\[
\{C_1, C_2, C_3 \ \text{blue}, \ C_4 \ \text{golden}\}; \quad
\{C_1, C_2, C_3 \ \text{golden}, \ C_4 \ \text{blue}\}; \quad
\{C_1, C_3 \ \text{blue}, \ C_2, C_4 \ \text{golden}\}.
\]
Obviously the trees with these colourings are not isomorphic and admit two automorphisms each.

\medskip
\noindent
\emph{The induction step.}  
Let $M_1, M_2, M_3$ be the desired colourings of a $k$-level tree.  
Consider the following colourings of the $(k+1)$-level tree:
\begin{itemize}
  \item $M_1$ for descendants of $B_1$ and $M_2$ for descendants of $B_2$;
  \item $M_2$ for descendants of $B_1$ and $M_3$ for descendants of $B_2$;
  \item $M_3$ for descendants of $B_1$ and $M_1$ for descendants of $B_2$.
\end{itemize}
It is quite obvious that these three colourings are not isomorphic and have the desired number of automorphisms.
\end{proof}

\medskip

\noindent
\textbf{Comment.}  
Note that in fact we solved the following problem: find a colouring of an $(n-2)$-level tree in 3 colours such that only the identical automorphism preserves the colours.  
Indeed, there are three mutually non-isomorphic colourings of a $3$-level tree in $2$ colours having only $2$ automorphisms.  
We want the colouring of the descendants of each vertex of level $n-2$ to be one of these colourings.  
The correspondence between vertices of level $n-2$ and these three colourings must be the desired colouring of the $(n-2)$-level tree admitting only the identical automorphism.

\end{tcolorbox}

\medskip
\medskip

\subsubsection{Detailed solution of Number Theory problems}

The second main group of problems from our benchmark are number theory problems. Similarly this problem was incorrectly solved by Gemini 2.5 Pro and correctly solved by GPT-5 Thinking.

\begin{tcolorbox}[breakable, colframe=blue!60!yellow, colback=blue!5!white, title=Sample Olympiad Number Theory Problem and Full Solution]
\textbf{Problem:} 
Find the greatest integer n, \(n > 10\)
 such that the remainder of n  when divided to each square between 2 and n/2 is an odd integer.

\newcommand{\thinline}{\noindent\rule{\linewidth}{0.2pt}}

\thinline

\textbf{Solution:}
The required number is 505. \\
\textbf{Example.} First, note that the remainder of $n$ when divided by $4$ is odd, hence $n$ is odd. Furthermore, observe that the quotient of $n$ when divided by a square less than $n/2$ is greater than or equal to $2$. On the other hand, the quotient of a division by an odd square cannot equal $3$, as the remainder would be even. Consequently, there are no positive integers $k$ such that $$3 \le \frac{n}{(2k-1)^2} < 4,$$ in other words, there is no $k \in \mathbb{N}$ with $\frac{n}{4} < (2k-1)^2 \le \frac{n}{3}$. Let $m \in \mathbb{N}^*$ so that $(2m-1)^2 \frac{n}{4} < \frac{n}{3} < (2m+1)^2$. Then $(2m+1)^2 - (2m-1)^2 > \frac{n}{3} - \frac{n}{4}$, hence $8m > \frac{n}{12}$. It follows that $(2m-1)^2 \le \frac{n}{4} < 24 \cdot m,$ so $m \in \{1, 2, \dots, 6\}$. Since $n < 96 \cdot m \le 576$, then the odd squares less than $n/2 < 288$ are
$9, 25, \dots, 225$. Recall that the quotients at the division by $9, 25, \dots, 225$ are even, so the quotients
at the division by $225$ and $169$ are both $2$ (else $4 \cdot 169 > 576$).
Thus $n = 450+a = 338+b$ with $0 < a < 225$, $0 < b \le 137$ and $a, b$ are odd, so $n \le 338+137 = 505$. For $n = 505$ one can easily check the claim.

\end{tcolorbox}

\medskip
\medskip

\subsubsection{Detailed solution of Graph Theory problems}

This problem is one that
Gemini 2.5 Pro failed to solve correctly however GPT-5 Thinking did solve correctly.

\medskip

\begin{tcolorbox}[breakable, colframe=blue!60!yellow, colback=blue!5!white, title=Sample Olympiad Graph Theory Problem and Full Solution]
\textbf{Problem:} 
Let \(G\) be a graph with \(9\) vertices.  
Suppose that for any chosen set of five vertices of \(G\), there are at least two edges whose endpoints are both contained in that five-vertex set.  
What is the minimum possible number of edges in \(G\)?

\newcommand{\thinline}{\noindent\rule{\linewidth}{0.2pt}}

\thinline

\textbf{Solution:}
The minimum is 9, achieved by three disjoint 3-cycles.\\
Let $a_n$ be the minimum number of edges in a graph on $n$ vertices satisfying the given condition.We show that $a_{n+1} \ge \frac{n+1}{n-1} a_n$. \\
Indeed,
given such a graph on $n+1$ vertices, let $l_i$ be the number of edges
of the graph obtained by removing vertex $i$ and all edges incident to
it. \\

Then $l_i \ge a_n$; on the other hand, $l_1 + \dots + l_{n+1} = (n-1)a_{n+1}$
since every edge is counted for every vertex except its endpoints.
The desired inequality follows.

Since $a_5 = 2$, we get $a_6 \ge 3$, $a_7 \ge 5$, $a_8 \ge 7$, $a_9 \ge 9$.

\end{tcolorbox}

\medskip
\medskip

\subsubsection{Detailed solution of Algebra problems}

This problem is one that
Gemini 2.5 Pro failed to solve correctly however GPT-5 Thinking did solve correctly.

\medskip

\begin{tcolorbox}[breakable, colframe=blue!60!yellow, colback=blue!5!white, title=Sample Olympiad Algebra Problem and Full Solution]
\textbf{Problem:} 
Determine the number of pairs of integers $(m,n)$ such that
\[
\sqrt{\,n + \sqrt{2016}\,} \;+\; \sqrt{\,m - \sqrt{2016}\,} \;\in \;\mathbb{Q}.
\]

\newcommand{\thinline}{\noindent\rule{\linewidth}{0.2pt}}

\thinline

\textbf{Solution:}
The answer is 1.\\
Let \[
r = \sqrt{n + \sqrt{2016}} + \sqrt{m - \sqrt{2016}}.
\]

Then

\[
n + m + 2 \sqrt{\,\bigl(n + \sqrt{2016}\bigr)\bigl(m - \sqrt{2016}\bigr)} \;=\; r^2
\]

and

\[
(m-n)\sqrt{2016} \;=\; \tfrac{1}{4}\,\bigl(r^2 - m - n\bigr)^2 \;-\; mn + 2016 \;\in\; \mathbb{Q}.
\]

Since \(\sqrt{2016} \notin \mathbb{Q}\), it follows that \(m = n\).

\[
\sqrt{n^2 - 2016} \;=\; \tfrac{1}{2}\,(r^2 - 2n) \;\in\; \mathbb{Q}.
\]

Hence, there is some nonnegative integer \(p\) such that
\[
n^2 - 2016 = p^2,
\]
and therefore
\[
2n + 2p = r^2.
\]

It follows that
\[
2(n+p) = r^2
\]
is the square of a rational and also an integer, hence a perfect square.

On the other hand,
\[
2016 = (n-p)(n+p),
\]
and \(n+p\) is a divisor of \(2016\), larger than \(\sqrt{2016}\).
Since \(n+p\) is even, so is also \(n-p\), and
\[
r^2 = 2(n+p)
\]
is a divisor of
\[
2016 = 2^5 \cdot 3^2 \cdot 7,
\]
larger than \(2\sqrt{2016} > 88\).
The only possibility is
\[
r^2 = 2^4 \cdot 3^2 = 12^2.
\]

Hence, \(n+p = 72\) and \(n-p = 28\), so we conclude that \(n = m = 50\).

\[
\boxed{\text{Thus, there is only one such pair.}}
\]

\end{tcolorbox}

\end{document}